\documentclass[arxiv]{imsart}

\RequirePackage[OT1]{fontenc}
\RequirePackage{amsthm,amsmath}
\RequirePackage[numbers]{natbib}
\RequirePackage[colorlinks,citecolor=blue,urlcolor=blue]{hyperref}

\usepackage{hyperref}       
\usepackage{url}            
\usepackage{booktabs}       
\usepackage{amsfonts}       
\usepackage{nicefrac}       
\usepackage{amsmath}
\usepackage{color}
\usepackage{amsthm}
\usepackage{graphicx}
\usepackage{caption}
\usepackage{subcaption}

\newtheorem{theorem}{Theorem}

\graphicspath{{figures/}}

\pubyear{2017}
\volume{0}
\issue{}
\firstpage{}
\lastpage{}

\startlocaldefs
\numberwithin{equation}{section}
\theoremstyle{plain}

\endlocaldefs

\begin{document}

\begin{frontmatter}
	\title{Poisson Intensity Estimation with Reproducing Kernels\thanksref{T1}}
\runtitle{Poisson Intensity Estimation with Reproducing Kernels}
	\thankstext{T1}{This work was supported by ERC (FP7/617071) and EPSRC (EP/K009362/1).}

\begin{aug}
\author{\fnms{Seth} \snm{Flaxman}\ead[label=e1]{flaxman@stats.ox.ac.uk},}
\author{\fnms{Yee Whye} \snm{Teh}\ead[label=e2]{y.w.teh@stats.ox.ac.uk}}
\and
\author{\fnms{Dino} \snm{Sejdinovic}\ead[label=e3]{dino.sejdinovic@stats.ox.ac.uk}}

\address{Department of Statistics\\
24-29 St Giles'\\
Oxford OX1 3LB\\
	United Kingdom \\
	\printead{e1,e2,e3}}

\runauthor{Flaxman et al.}

\affiliation{University of Oxford}

\end{aug}

\begin{abstract}
	Despite the fundamental nature of the inhomogeneous Poisson process in the theory and application of stochastic processes, and its attractive generalizations (e.g.~Cox process), few tractable nonparametric modeling approaches of intensity functions exist, especially when observed points lie in a high-dimensional space. In this paper we develop a new, computationally tractable Reproducing Kernel Hilbert Space (RKHS) formulation for the inhomogeneous Poisson process. We model the square root of the intensity as an RKHS function. 
	Whereas RKHS models used in supervised learning rely on the so-called representer theorem, the form of the inhomogeneous Poisson process likelihood means that the representer theorem does not apply. 
	However, we prove that the representer theorem does hold in an appropriately transformed RKHS, guaranteeing that the optimization of the penalized likelihood can be cast as a tractable finite-dimensional problem.  The resulting approach is simple to implement, and readily scales to high dimensions and large-scale datasets.
\end{abstract}


\tableofcontents
\end{frontmatter}

\section{Introduction}
Poisson processes are ubiquitous in statistical science, with a long history
spanning both theory (e.g.~\cite{kingman1993poisson}) and
applications (e.g.~\cite{diggle2013spatial}), especially
in the spatial statistics and time series literature. Despite their
ubiquity, fundamental questions in their application to
real datasets remain open. Namely, scalable nonparametric models for intensity functions of inhomogeneous Poisson processes are not well understood, especially in multiple dimensions since the standard approaches, based on kernel smoothing, are akin to density estimation and hence scale poorly with dimension. In this contribution, we propose a step towards such scalable nonparametric modeling and introduce a new Reproducing Kernel
Hilbert Space (RKHS) formulation for inhomogeneous Poisson process modeling, which is
based on the Empirical Risk Minimization (ERM) framework. We model the square root of the intensity as an RKHS function and consider a risk functional given by a penalized version of the inhomogeneous Poisson process
likelihood. However, standard representer theorem arguments do not apply directly due to the form of the likelihood. Namely, the fundamental difference arises since the observation that {\em no points} occur in some region is just as important as the locations of the points
that do occur. Thus, the likelihood depends not only on the evaluations of the intensity at the observed points, but also on its integral across
the domain of interest. As we will see, this difficulty can be overcome by appropriately adjusting the RKHS under consideration. We prove a version of the representer theorem in this adjusted RKHS, which coincides with the original RKHS as a space of functions but has a different inner product structure.
This allows us to cast the estimation problem as an optimization over a finite-dimensional subspace of the adjusted RKHS.
The derived method is demonstrated to give better performance than a na\"ive unadjusted RKHS method which resorts to an
optimization over a subspace without representer theorem guarantees. We describe cases where adjusted RKHS can be described with explicit
Mercer expansions and propose numerical approximations where Mercer expansions are not available. We observe strong performance of the proposed method on a variety of synthetic, environmental, crime and bioinformatics data. 

\section{Background and related work}
\subsection{Poisson process}
We briefly state relevant definitions for point processes over domains $S \subset \mathbb{R}^D$, following \cite{cressie2011}.
For Lebesgue
measurable subsets $T \subset S$, $N(T)$ denotes the number of events
in $T \subset S$. $N(\cdot)$ is a stochastic process characterizing
the point process. Our focus is on providing a nonparametric estimator for
the first-order intensity of a point process, which is defined as:
\begin{equation}
	\lambda(s) = \lim_{|ds|\rightarrow 0} \mathbb{E}[N(ds))]/|ds|.
\end{equation}
The inhomogeneous Poisson process is driven solely by the intensity
function $\lambda(\cdot)$:
\begin{equation}
	N(T) \sim \mbox{Poisson}(\int_T \lambda(x) dx).
\end{equation}
In the homogeneous Poisson process, $\lambda(x) = \lambda$ is constant, so the number of points in any region $T$ simply depends on the volume of $T$, which we denote $|T|$:
\begin{equation}
	N(T) \sim \mbox{Poisson}(\lambda |T|).
\end{equation}
For a given intensity function $\lambda(\cdot)$, 
the likelihood of a set of $N=N(S)$ points $x_1, \ldots, x_N$ observed over a domain $S$ is given by:
\begin{equation}
	\mathcal{L}(x_1,\ldots,x_N | \lambda(\cdot)) = \prod_{i=1}^N \lambda(x_i) e^{-\int_S \lambda(x) dx}
		\label{eq:inhomogeneous-poisson}
\end{equation}

\subsection{Reproducing Kernel Hilbert Spaces}
Given a non-empty domain $S$ and a positive definite kernel function $k:S\times S\to\mathbb R$, there exists a unique reproducing kernel Hilbert space (RKHS) $\mathcal H_k$. An RKHS is a space of functions $f:S \to\mathbb R$, in which evaluation is a continuous functional, meaning it can be represented by an inner product $f(x)=\langle f,k(x,\cdot) \rangle_{\mathcal H_k}$ for all $f\in\mathcal H_k, x\in S$ (this is known as the reproducing property), cf.~\citet{BerTho04}. While $\mathcal H_k$ is in most interesting cases an infinite-dimensional space of functions, due to the classical representer theorem \citep{KIMELDORF1971}, \citep[Section 4.2]{scholkopf2002learning}, optimization over $\mathcal H_k$ is typically a tractable finite-dimensional problem. In particular, if we have a set of $N$ observations $x_1,\ldots,x_N$, $x_i\in S$ and consider the problem:
\begin{equation}
 \min_{f\in\mathcal H_k} \left\{ R\left(f(x_1),\ldots,f(x_N)\right) + \Omega\left( \|f \|_{\mathcal{H}_k}\right)\right\}.
 \label{eq:general_objective}
\end{equation}
where $R\left(f(x_1),\ldots,f(x_N)\right)$ depends on $f$ through its evaluations on the set of observations only, and $\Omega$ is a non-decreasing function of the RKHS norm of $f$, there exists a solution to Eq.~\eqref{eq:general_objective} of the form $f^*(\cdot) = \sum_{i=1}^N \alpha_i k(x_i,\cdot)$, and the optimization can thus be cast in terms of the so-called \emph{dual coefficients} $\alpha\in\mathbb R^N$. This formulation is widely used in the framework of regularized Empirical Risk Minimization (ERM) for supervised learning, where $R\left(f(x_1),\ldots,f(x_N)\right)=\frac{1}{N}\sum_{i=1}^N L(f(x_i),y_i)$ is the empirical risk corresponding to a loss function $L$, e.g.~squared loss for regression, logistic or hinge loss for classification.

If domain $S$ is compact and kernel $k$ is continuous, one can assign to $k$ its integral kernel operator $T_k:L_2(S)\to L_2(S)$, given by $T_k g=\int_S k(x,\cdot)g(x)dx$, which is positive, self-adjoint and compact. There thus exists an orthonormal set of eigenfunctions $\{e_j\}_{j=1}^\infty$ of $T_k$ and the corresponding eigenvalues $\{\eta_j\}_{j=1}^\infty$, with $\eta_j \rightarrow 0$ as $j \rightarrow \infty$.
This spectral decomposition of $T_k$ leads to Mercer's representation of kernel function $k$ \cite[Section 2.2]{scholkopf2002learning}:
\begin{equation}
	k(x,x') = \sum_{j=1}^\infty \eta_j e_j(x) e_j(x'), \qquad x,x'\in S
	\label{eq:mercer}
\end{equation}
with uniform convergence on $S\times S$. Any function $f\in\mathcal H_k$ can then be written as $f=\sum_j b_j e_j$ where $\|f \|_{\mathcal{H}_k}^2=\sum_j b_j^2/\eta_j<\infty$.

Note that above we have focused on Mercer expansion with respect to the
Lebesgue measure, but other base measures are also often considered in
literature, e.g.~\cite[section 4.3.1]{rasmussen2006gaussian}.

\subsection{Related work}
The classic approach to nonparametric intensity estimation
is based on smoothing kernels \citep{ramlau-hansen1983,diggle1985kernel}
and has a form closely related to the kernel density estimator:
\begin{equation}
\hat\lambda(x) = \sum_{i=1}^N \kappa(x_i-x)
\end{equation}
where $\kappa$ is a smoothing kernel (related to but distinct from the RKHS kernels described in the previous section), that is, any bounded function integrating
to $1$.  Early work in this area focused on edge-corrections and methods for choosing
the bandwidth \citep{diggle1985kernel,berman1989estimating,brooks1991asymptotic}. Connections with RKHS
have been considered by, for example, \citet{bartoszynski1981some} who use a maximum penalized
likelihood approach based on Hilbert spaces to estimate the intensity of a Poisson process.
There is long literature on maximum penalized likelihood approaches to density
estimation, which also contain interesting connections with RKHS, e.g.~\cite{silverman1982}.

Much recent work on estimating intensities for point processes has focused on
Bayesian approaches to modeling Cox processes. The log Gaussian Cox Process
\citep{moller1998log} and related parameterizations of Cox (doubly stochastic)
Poisson processes in terms of Gaussian processes have been proposed, along with
Monte Carlo \citep{adams2009tractable,diggle2013spatial,teh2011gaussian},
Laplace approximation \citep{illian2012toolbox,cunningham2008fast,flaxman2015fast}
and variational \citep{lloyd2015variational,yves2015scalable} inference schemes.

Another related body of literature concerns Cox processes with intensities parameterized as the sum of squares
of $k$ Gaussian processes, called the permanent process \cite{mccullagh2006permanental}. Interestingly, calculating
the density of the permanent process relies on a kernel transformation similar to the one we propose below. Unlike
these approaches, however, we are not working in a doubly stochastic (Cox process) framework; rather we are taking
a penalized maximum likelihood estimation perspective to estimate the intensity of an inhomogeneous Poisson process.
As future work, it would be worthwhile to explore deeper connections between our formulation and the permanent process, 
e.g.~by considering an RKHS formulation of Cox processes or by considering an inhomogeneous Poisson process whose intensity
is the sum of squares of functions in an RKHS.

\section{Proposed method and kernel transformation}
Let $S$ be a compact domain of observations.
Let $k:S\times S \to \mathbb R$ be a continuous positive definite kernel, and $\mathcal{H}_k$ its corresponding RKHS of functions $f:S\to \mathbb R$. We model the intensity function $\lambda(\cdot)$ of an inhomogeneous Poisson process
as:
\begin{equation}
	\lambda(x) := a f^2(x),\quad x\in S,
	\label{eq:parametrization}
\end{equation}
which is parameterized by $f \in \mathcal{H}_k$ and an additional scale parameter $a>0$. 
The flexibility of choosing $k$ means that we can encode structural assumptions of our domain,
e.g.~periodicity in time or periodic boundary conditions (see Section \ref{sec:sobolev}).
Note that we have squared $f$ to ensure that the intensity
is non-negative on $S$, a pragmatic choice that has previously appeared in the literature (e.g.~\cite{lloyd2015variational}).
While we lose identifiability (since $f$ and $-f$ are equivalent), as shown below we end up with a finite dimensional, and thus tractable, optimization problem.

The rationale for including $a$ is
that it allows us to decouple the overall scale and units of the intensity (e.g.~number of points per hour versus number of points per year) from the
penalty on the complexity of $f$ which arises from the classical regularized Empirical Risk Minimization framework (and which should depend only on how complex, i.e.~``wiggly'' $f$ is).

We use the inhomogeneous Poisson process likelihood from Eq.~\eqref{eq:inhomogeneous-poisson} to write the log-likelihood of a Poisson process
corresponding to the observations $\{x_1, \ldots, x_N\}$, for $x_i \in S$, and intensity $\lambda(\cdot)$:
\begin{equation}
\ell(x_1, \ldots, x_N | \lambda) = \sum_{i=1}^N \log(\lambda(x_i)) - \int_{S} \lambda(x) dx.
\label{eq:log-likelihood}
\end{equation}
We will consider the problem of minimization of the penalized negative log likelihood, where the regularization term corresponds to the squared Hilbert space norm of $f$ in parametrization Eq.~\eqref{eq:parametrization}:
\begin{equation}
 \min_{f\in\mathcal H_k} \left\{-\sum_{i=1}^N \log(af^2(x_i)) + a\int_{S} f^2(x) dx + \gamma \|f \|_{\mathcal{H}_k}^2\right\}.
 \label{eq:objective}
\end{equation}
This objective is akin to a classical regularized empirical risk minimization framework over RKHS: there is a term that depends on evaluations of $f$ at the observed points $x_1,\ldots,x_N$ as well as a term corresponding to the RKHS norm. However, the representer theorem does not apply directly to Eq.~\eqref{eq:objective}: since there is also a term given by the $L_2$-norm of $f$, there is no guarantee that there is a solution of Eq.~\eqref{eq:objective} that lies in $\text{span}\{k(x_i,\cdot)\}_{i=1}^{N}$. We will show that Eq.~\eqref{eq:objective} fortunately still reduces to a finite-dimensional optimization problem corresponding to a different kernel function $\tilde k$ which we define below.

Using the Mercer expansion of $k$ in Eq.~\eqref{eq:mercer}, we can write the objective Eq.~\eqref{eq:objective} as follows:
\begin{align}
	J\left[f\right] & = -\sum_{i=1}^N \log(a f^2(x_i)) + a \|f\|^2_{L_2(S)} + \gamma \|f \|_{\mathcal{H}_k}^2 \\
   & = -\sum_{i=1}^N \log(a f^2(x_i)) + a \sum_{j=1}^{\infty} b_j^2  + \gamma \sum_{j=1}^{\infty} \frac{b_j^2}{\eta_j}.
\end{align}
The last two terms can now be merged together, giving
\begin{align*}
	& a \sum_{j=1}^{\infty} b_j^2  + \gamma \sum_{j=1}^{\infty} \frac{b_j^2}{\eta_j}
	 = \sum_{j=1}^{\infty} b_j^2 \frac{a \eta_j  + \gamma}{\eta_j}
	 = \sum_{j=1}^{\infty} \frac{b_j^2}{\eta_j (a \eta_j  + \gamma)^{-1}}.
\end{align*}
Now, if we define kernel $\tilde k$ to be the kernel corresponding to the integral operator $T_{\tilde k}:=T_k(a T_k + \gamma I)^{-1}$, i.e., $\tilde k$ is given by:
\begin{equation*}
	\tilde k(x,x') = \sum_{j=1}^\infty \frac{\eta_j}{a\eta_j+\gamma} e_j(x) e_j(x'), \qquad x,x'\in S,
\end{equation*}
we see that:
\begin{equation}
	J\left[f\right] = -\sum_{i=1}^N \log(a f^2(x_i)) + \|f \|_{\mathcal{H}_{\tilde k}}^2.
	\label{eq:Jtildenorm}
\end{equation}

Thus, we have merged the two squared norm terms into a squared norm in a new RKHS. We note that a similar idea has previously been used to modify Gaussian process priors in \cite{Csato2001}, albeit in a different context, and that a similar transformation appears in the expression
for the distribution of a permanent process \cite{mccullagh2006permanental}. We are now ready to state the representer theorem in terms of kernel $\tilde k$.
\begin{theorem}{There exists a solution of 
    Eq.~\eqref{eq:objective}
    for observations $x_1, \ldots, x_N$, which takes the form $f^*(\cdot) = \sum_{i=1}^N \alpha_i \tilde k(x_i,\cdot)$.}
	\end{theorem}
	\begin{proof}
	 Since $\sum_j \frac{b_j^2}{\eta_j}<\infty$ if and only if $\sum_j \frac{b_j^2}{\eta_j (a \eta_j  + \gamma)^{-1}}<\infty$, i.e. $f \in \mathcal{H}_k \iff f \in \mathcal{H}_{\tilde k}$, we have that the two spaces correspond to exactly the same set of functions. Optimization over $\mathcal H_k$ is therefore equivalent to optimization over $\mathcal H_{\tilde k}$. 
		
		The proof now follows by applying the classical representer theorem in $\tilde k$ to the representation of the objective function in Eq.~\eqref{eq:Jtildenorm}.  We decompose $f \in \mathcal{H}_{\tilde k}$
as the sum of two functions:
\begin{equation} 
	f(\cdot) = \sum_{j=1}^N \alpha_j \tilde k(x_j,\cdot) + v
\end{equation}
		where $v$ is orthogonal in $\mathcal{H}_{\tilde k}$ to the span of $\{\tilde k(x_j,\cdot)\}_j$. We prove that the first term in the objective $J[f]$ given in Eq.~\eqref{eq:Jtildenorm}, $-\sum_{i=1}^N \log(af^2(x_i))$, is independent of $v$. It depends
on $f$ only through the evaluations $f(x_i)$ for all $i$. Using the reproducing property we have:
		\begin{equation} f(x_i) = \langle f, \tilde k(x_i,\cdot) \rangle_{\mathcal{H}_{\tilde k}} = \sum_j \alpha_j \tilde k(x_j,x_i) + \langle v, \tilde k(x_i,\cdot) \rangle_{\mathcal{H}_{\tilde k}} = \sum_j \alpha_j \tilde k(x_j,x_i)
\end{equation}
where the last step is by orthogonality.
Next we substitute into the regularization term:
\begin{align}
	& \gamma \| \sum_j \alpha_j \tilde k(x_j,\cdot) + v \|_{\mathcal{H}_{\tilde k}}^2 
	 = \gamma \| \sum_j \alpha_j \tilde k(x_j,\cdot)\|^2_{\mathcal{H}_{\tilde k}} + \| v \|_{\mathcal{H}_{\tilde k}}^2  
	 \geq \gamma \| \sum_j \alpha_j \tilde k(x_j,\cdot)\|^2_{\mathcal{H}_{\tilde k}}. 
\end{align}
Thus, the choice of $v$ has no effect on the first term in $J[f]$ and a non-zero $v$ can only increase the second term $\|f \|_{\mathcal{H}_{\tilde k}}^2$, so we conclude that $v=0$ and that
$f^*=\sum_{j=1}^N \alpha_j \tilde k(x_j,\cdot)$ is the minimizer.
	\end{proof}
{\bf Remark 1.} The notions of the inner product in $\mathcal{H}_k$ and $\mathcal{H}_{\tilde k}$ are different and thus in general $\text{span}\{k(x_i,\cdot)\}\neq\text{span}\{\tilde k(x_i,\cdot)\}$.

{\bf Remark 2.} Notice that unlike in a standard ERM setting, $\gamma = 0$ does not recover the unpenalized risk, because $\gamma$ appears in $\tilde k$.
Notice further that the overall scale parameter $a$ also appears in $\tilde k$. This is important in practice, because it allows us to decouple the
scale of the intensity (which is controlled by $a$) from its complexity (which is controlled by $\gamma$).

{\bf Illustration.} The eigenspectrum of $\tilde k$ where $k$ is a squared exponential kernel is shown in Figure \ref{fig:ktilde} for various settings of $a$ and $\gamma$. Reminiscent of
spectral filtering studied by \citet{Muandet2014_spectral}, in the top plot we see that depending on the settings of $a$ and $\gamma$, eigenvalues of $\tilde k$ are shrunk or inflated as compared to $k(x,x')$ which is shown in black.
In the bottom plot, the values of $k(0,x)$ are shown for the same set of kernels.

\begin{figure}
  \includegraphics[width=.4\paperwidth]{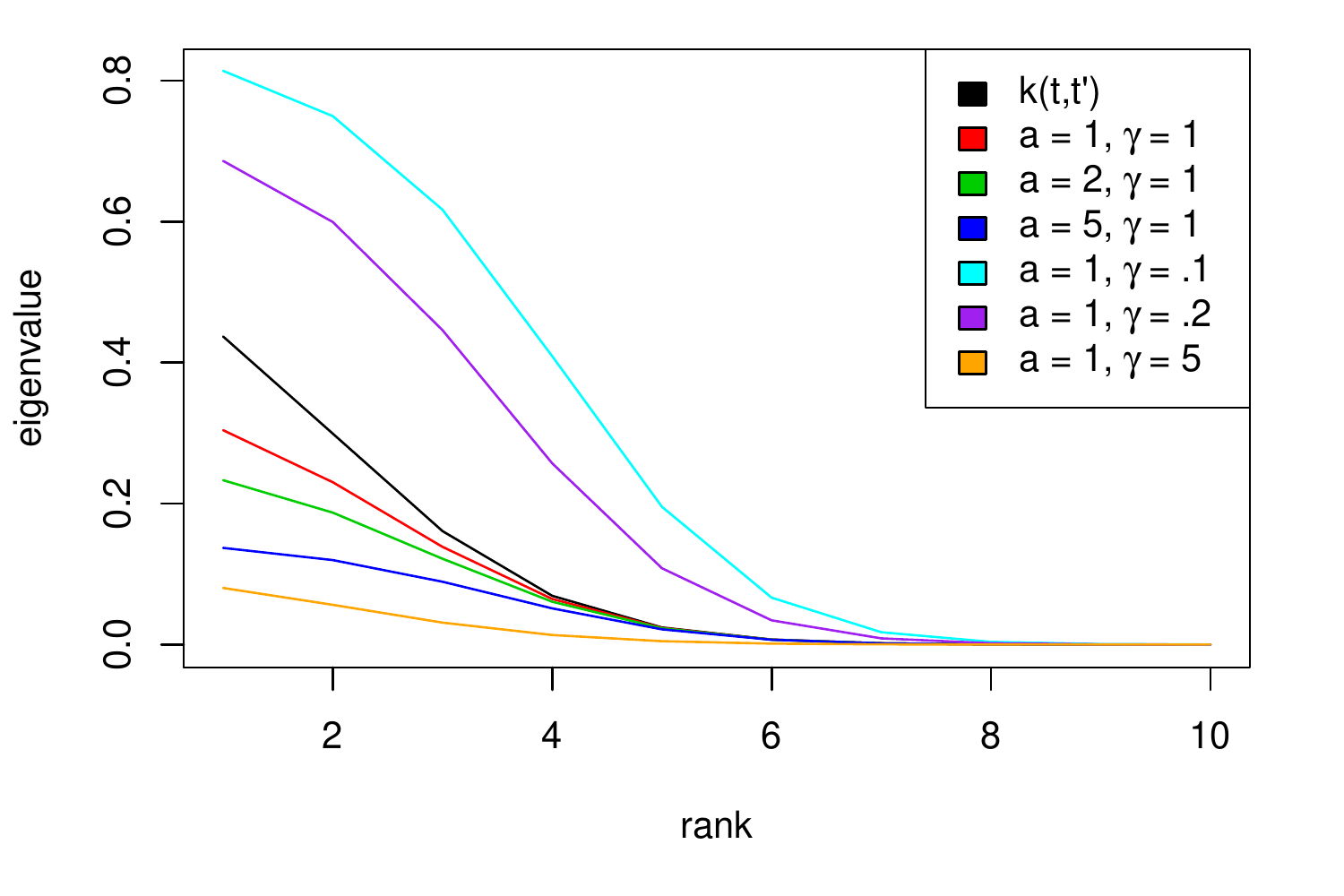}

  \includegraphics[width=.4\paperwidth]{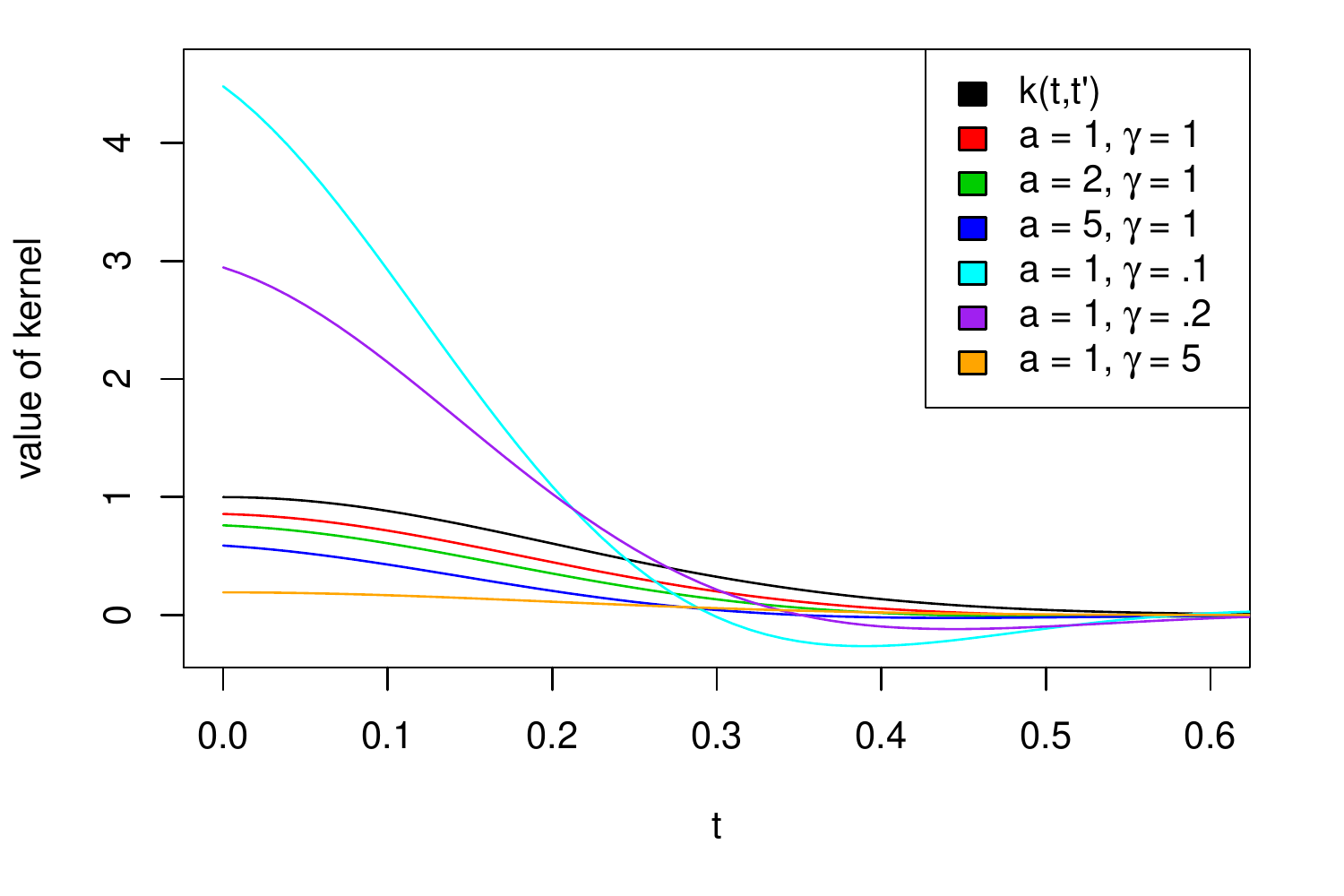}
	\caption{Eigenspectrum of $\tilde k$ (top) and values of $\tilde k$ (bottom) for various settings of $a$ and $\gamma$.}
	\label{fig:ktilde}
\end{figure}

\section{Computation of $\tilde k$}
In this section, we consider first the case in which an explicit Mercer expansion is known,
and then we consider the more commonly encountered situation in which we only have access to
the parametric form of the kernel $k(x,x')$, so we must approximate $\tilde k$. We
show experimentally that our approximation is very accurate by considering the Sobolev kernel,
which can be expressed in both ways.

\subsection{Explicit Mercer Expansion}
We start by assuming that we have a kernel $k$ with an explicit Mercer expansion with respect to a base
measure of interest (usually the Lebesgue measure on $S$), so we have
eigenvectors $\{e_j(x)\}_{j\in J}$ and eigenvalues $\{\eta_j\}_{j\in J}$:
\begin{equation}
	k(x,x') = \sum_{j\in J} \eta_j e_j(x) e_j(x'),
\end{equation}
with an at most countable index set $J$.
Given $a$ and $\gamma$ we can calculate:
\begin{equation}
\label{eq:adjusted}
	\tilde k(x,x') = \sum_{j\in J} \frac{\eta_j}{a \eta_j + \gamma} e_j(x) e_j(x')
\end{equation}
up to a desired precision as informed by the spectral decay in $\{\eta_j\}_{j\in J}$. 
Below we consider kernels for which explicit Mercer expansions are known: a kernel on the Sobolev space $[0,1]$ with a periodic boundary condition, the squared exponential kernel, and the Brownian bridge kernel.
We also show how our formulation can be extended to multiple dimensions using a tensor product formulation. 
Although not practical for large datasets, the Mercer
expansions given below, summing terms up to $j = 50$ (for which the error is less than $10^{-5}$), can be used to 
evaluate approximations for the cases in which Mercer expansions are not available.

\subsubsection{Sobolev space on $[0,1]$ with a periodic boundary condition}
\label{sec:sobolev}
We consider a kernel on the Sobolev space on $[0,1]$ with a periodic boundary condition, proposed by \citet[chapter 2]{wahba1990spline}
and recently used in \citet{bach2015equivalence}.
The kernel is given by:
\begin{eqnarray*}
k(x,y) & = & 1 + \sum_{m=1}^{\infty}\frac{2\cos\left(2\pi m\left(x-y\right)\right)}{(2\pi m)^{2s}}\\
 & = & 1+ \sum_{m=1}^{\infty}\frac{2}{(2\pi m)^{2s}}\left[\cos\left(2\pi mx\right)\cos\left(2\pi my\right)+\sin\left(2\pi mx\right)\sin\left(2\pi my\right)\right],\\
 & = & 1+\frac{(-1)^{s-1}}{(2s)!}B_{2s}(\{x-y\}),
	\label{eq:sobolev-kernel}
 \end{eqnarray*}
where $s=1,2,\ldots$ denotes the order of the Sobolev space and $B_{2s}(\{x-y\})$ is the Bernoulli polynomial of degree $2s$ applied to the fractional part of $x-y$. The corresponding RKHS is the space of functions on $[0,1]$ with absolutely continuous $f,f',\ldots,f^{(s-1)}$ and square integrable $f^{(s)}$ satisfying a periodic boundary condition $f^{(l)}(0)=f^{(l)}(1)$, $l=0,\ldots,s-1$. For more details, see \cite[Chapter 2]{wahba1990spline}.

Bernoulli polynomials admit a simple form for low degrees. In particular,
\begin{eqnarray*}
 B_2(t)&=&t^2-t+\frac{1}{6},\\
 B_4(t)&=&t^4-2t^3+t^2-\frac{1}{30},\\
 B_6(t)&=&t^6-3t^5+\frac{5}{2}t^4-\frac{1}{2}t^2+\frac{1}{42}.
\end{eqnarray*}

Moreover, note that:
\begin{eqnarray*}
\int_{0}^{1}2\cos\left(2\pi mx\right)\sin\left(2\pi m'x\right)dx & = & 0,\\
\int_{0}^{1}2\cos\left(2\pi mx\right)\cos\left(2\pi m'x\right)dx & = & \delta(m-m'),\\
\int_{0}^{1}2\sin\left(2\pi mx\right)\sin\left(2\pi m'x\right)dx & = & \delta(m-m').
\end{eqnarray*}
Thus, the desired Mercer expansion (with respect to the Lebesgue measure) is given by $k(x,y)=\sum_{m\in\mathbb Z} \eta_m e_m(x)e_m(y)$ with eigenfunctions $e_0(x)=1$ and for $m=\{1,2,\ldots\}$, $e_m(x) = \sqrt{2}\cos\left(2\pi mx\right)$, $e_{-m}(x)=\sqrt{2}\sin\left(2\pi mx\right)$
and corresponding eigenvalues $\eta_0=1$, $\eta_{m}=\eta_{-m}=(2\pi m)^{-2s}$.

Now, the adjusted kernel $\tilde k(x,y)$ from \eqref{eq:adjusted} is given by
 \begin{eqnarray*}
  \tilde k(x,y) &=& \sum_{m\in\mathbb Z} \frac{\eta_m}{\eta_m+\gamma}e_m(x)e_m(y)\\
	      {}&=&\frac{1}{1+\gamma}+ \sum_{m=1}^{\infty}\frac{2\cos\left(2\pi m\left(x-y\right)\right)}{1+\gamma(2\pi m)^{2s}}.
 \end{eqnarray*}

\subsubsection{Squared exponential kernel}
\label{section:se}
A Mercer expansion for the squared exponential kernel was proposed in \cite{zhu1997gaussian} and refined in \cite{fasshauer2012stable}. However, this expansion is with respect to a Gaussian measure on $\mathbb R$, i.e., it consists of eigenfunctions which form an orthonormal set in $L^2(\mathbb R,\nu)$ where $\nu=\mathcal N(0,\ell^2I)$. The formalism can therefore be used to estimate Poisson intensity functions with respect to such Gaussian measure. In the classical framework, where the intensity is with respect to a Lebesgue measure, numerical approximations of Mercer expansion, as described in Section \ref{sec:numerical} are needed.
Following the exposition in \cite[section 4.3.1]{rasmussen2006gaussian}
and the relevant errata\footnote{http://www.gaussianprocess.org/gpml/errata.html} we parameterize the kernel as:
\begin{equation}
	k(x,x') = \exp(-\frac{\|x-x'\|^2}{2 \sigma^2})
\end{equation}
	The Mercer expansion with respect to $\nu=\mathcal N(0,\ell^2I)$ then has the eigenvalues
\begin{equation}
	\eta_i = \sqrt{\frac{2a}{A}} B^i,
\end{equation}
and eigenfunctions
\begin{equation}
	e_i(x) = \frac{1}{\sqrt{\sqrt{a/c}~2^i i!}} \exp(-(c-a)x^2) H_i(\sqrt{2c}x)
\end{equation}
where $H_i$ is the $i$-th order (physicist's) Hermite polynomial, $a = \frac{1}{4\sigma^2}$, $b = \frac{1}{2\ell^2}$, 
$c = \sqrt{a^2 + 2ab}$, $A = a + b + c$, and $B = b/A$.
Thus we have the following eigenvalues for $\tilde k$:
\begin{equation}
	\widetilde{\eta_i} = \frac{\eta_i}{a \eta_i + \gamma} = \frac{1}{a + \gamma \sqrt{\frac{A}{2a}} B^{-i}}
\end{equation}
while the eigenfunctions remain the same.

\subsubsection{Brownian Bridge kernel}

This is the kernel on $[0,1]$, given by
$$k(x,y)=\min(x,y)-xy=\sum_{m=1}^\infty \frac{2\sin(\pi m x)\sin(\pi m y)}{\pi^2 m^2},$$
with the eigenvalues and eigenfunctions in the Mercer expansion with respect to Lebesgue measure

\begin{equation}
\eta_{m}=\frac{1}{\pi^{2}m^{2}},\;e_{m}(x)=\sqrt{2}\sin\left(\pi m x\right),\quad m=1,2,\ldots.
\end{equation}

Thus one can form

\begin{eqnarray*}
  \tilde k(x,y) &=& \sum_{m=1}^\infty \frac{\eta_m}{\eta_m+c}e_m(x)e_m(y)\\
	      {}&=&\sum_{m=1}^{\infty}\frac{2\sin(\pi m x)\sin(\pi m y)}{1+c\pi^2m^2}.
 \end{eqnarray*}
 
The functions in the corresponding RKHS are pinned to zero at both ends of the segment.

\subsubsection{Extending the Mercer expansion to multiple dimensions}
\label{section:mercer-tensor}
The extension of any kernel to higher dimensions can be constructed by considering tensor product spaces: $\mathcal{H}_{k_1 \otimes k_2}$ (where
$k_1$ and $k_2$ could potentially be different kernels with different hyperparameters).
If $k_1$ has eigenvalues $\eta_i$ and eigenfunctions $e_i$ and $k_2$ has eigenvalues $\delta_j$ and eigenfunctions $f_j$, 
then the eigenvalues of the product space are then given by the Cartesian product $\eta_i \delta_j, \forall i,j$, and similarly
the eigenfunctions are given by $e_i(x)f_j(y)$. 
Our regularized kernel has the following Mercer expansion:
\begin{equation}
\widetilde{k_1 \otimes k_2}((x,y),(x',y')) = \sum_{ij} \frac{\eta_i \delta_j}{a \eta_i \delta_j + \gamma} e_i(x) e_i(x') f_j(y) f_j(y')
\end{equation}
Notice that $\widetilde{k_1 \otimes k_2}$ is the kernel corresponding to the integral operator $(T_{k_1} \otimes T_{k_2})(a T_{k_1} \otimes T_{k_2} + \gamma I)^{-1}$ which is different than $\tilde k_1 \otimes \tilde k_2$.

Notice that this approach does not lead to a method that scales well in high dimensions, which is further motivation for the approximations
developed below.

\subsection{Numerical approximation when Mercer expansions are not available}
\label{sec:numerical}
We propose an approximation to $\tilde k$ given access only to a kernel $k$ for which we do not have an explicit Mercer expansion with respect to Lebesgue measure. We only assume that we can form Gram matrices corresponding to $k$ and calculate their eigenvectors and eigenvalues. As a side benefit, this representation will also enable scalable computations through Toeplitz / Kronecker algebra \cite{cunningham2008fast,gilboa2013scaling,flaxman2015fast} or primal reduced rank approximations \cite{williams2001using}.

Let us first consider the one-dimensional case and construct a uniform grid ${\bf u}=(u_1,\ldots,u_m)$ on $[0,1]$. Then the integral
kernel operator $T_k$ can be approximated with the (scaled) kernel matrix
$\frac 1 m K_{\bf uu}:\mathbb R^m\to\mathbb R^m$, where $\left[K_{\bf uu}\right]_{ij}=k(u_i,u_j)$, and thus $\tilde K_{\bf uu}$ is
approximately $K_{\bf uu}\left(\frac{a}{m}K_{\bf uu}+\gamma I\right)^{-1}$. Note that for the general case
of multidimensional domains $S$, the kernel matrix would have to be multiplied by $\mbox{vol}(S)$. Without
loss of generality we assume $\mbox{vol}(S) = 1$ below.

We are not primarily interested in evaluations of $\tilde k$ on this grid,
but on the observations $x_1, \ldots, x_N$. Simply adding the observations into the kernel
matrix is not an option however, as it changes the base measure with respect to
which the integral kernel operator is to be computed (Lebesgue measure on
$[0,T]$). Thus, we consider the relationship between the eigendecomposition of $K_{\bf uu}$ and
the eigenvalues and eigenfunctions of the integral kernel operator $T_k$.

\setlength{\abovedisplayskip}{3pt}
\setlength{\belowdisplayskip}{3pt}
Let $\lambda_i^u, {\bf e}_i^u$ be the eigenvalue/eigenvector pairs of the matrix $K_{\bf{uu}}$, i.e., its eigendecomposition
is given by $K_{\bf{uu}}=Q\Lambda Q^\top =\sum_{i=1}^m \lambda_i^u{\bf e}_i^u({\bf e}_i^u)^\top$. Then the estimates of the eigenvalues/eigenfunctions
of the integral operator $T_k$ are given by the Nystr\"om method (see \citet[Section 4.3]{rasmussen2006gaussian} and references therein, especially
\citet{baker1977numerical}):
\begin{equation}
 \hat\eta_i = \frac{1}{m}\lambda_i^u,\qquad \hat e_i(x)=\frac{\sqrt m}{\lambda_i^u}K_{x{\bf u}}{\bf e}_i^u,
\end{equation}
with $K_{x{\bf u}}=\left[k(x,u_1),\ldots,k(x,u_m)\right]$, leading to:
\begin{eqnarray}
 \widehat{\tilde k}(x,x')&=&\sum_{i=1}^m\frac{\hat\eta_i}{a\hat\eta_i+\gamma}\hat e_i(x)\hat e_i(x') \\\nonumber
		     {}&=&\sum_{i=1}^m\frac{\frac{1}{m}\lambda_i^u}{\frac{a}{m}\lambda_i^u+\gamma}\cdot\frac{m}{(\lambda_i^u)^2}K_{x{\bf u}}{\bf e}_i^u({\bf e}_i^u)^\top K_{{\bf u}x'}\\\nonumber
		     {}&=&K_{x{\bf u}}\left\{\sum_{i=1}^m\frac{1}{\left(\frac{a}{m}\lambda_i^u+\gamma\right)\lambda_i^u}{\bf e}_i^u({\bf e}_i^u)^\top \right\}K_{{\bf u}x'}.
\end{eqnarray}

For an estimate of the whole matrix $\tilde K_{\bf xx}$ we thus have
\begin{eqnarray}
\label{eq:tildeKestimator}
 \widehat{\tilde K}_{\bf xx}&=&K_{{\bf xu}}\left\{\sum_{i=1}^m\frac{1}{\left(\frac{a}{m}\lambda_i^u+\gamma\right)\lambda_i^u}{\bf e}_i^u({\bf e}_i^u)^\top \right\}K_{{\bf ux}} {}\nonumber\\
 &=&K_{{\bf xu}}Q\left(\frac{a}{m}\Lambda^2+\gamma \Lambda\right)^{-1}Q^\top K_{{\bf ux}}.
\end{eqnarray}

The above is reminiscent of the Nystr\"om method \citep{williams2001using} proposed for speeding up Gaussian process regression. It has
computational cost $O(m^3+N^2m)$.  A reduced rank representation for Eq.~\eqref{eq:tildeKestimator} is straightforward by considering only the top $p$ eigenvalues/eigenvectors of $K_{\bf{uu}}$.
Furthermore, a primal representation with the features corresponding to kernel $\tilde k$ is readily available and is given by
\begin{equation}
 \tilde \phi(x) = \left(\frac{a}{m}\Lambda^2+\gamma \Lambda\right)^{-1/2}Q^\top K_{{\bf u}x},
\end{equation}
which allows linear computational cost in the number $N$ of observations.

For $D>1$ dimensions, one can exploit Kronecker and Toeplitz algebra approaches. Assuming that the $K_{\bf uu}$ matrix corresponds to a Cartesian product structure of the one-dimensional grids of size $m$, one can write
$K_{\bf uu}= K_1 \otimes K_2 \cdots \otimes K_D$. Thus, the eigenspectrum can be efficiently calculated by eigendecomposing each of the smaller $m\times m$ matrices $K_1, \ldots, K_D$ and then applying standard Kronecker algebra, thereby avoiding ever having
to form the prohibitively large $m^D\times m^D$ matrix $K_{\bf uu}$. For regular grids and stationary kernels, each small matrix will be Toeplitz structured, yielding further efficiency gains \citep{wilson2015thoughts}. The resulting approach thus scales linearly in dimension $D$.

An even simpler alternative to the above
is to sample the points $u_1, \ldots, u_m$ uniformly from the domain $S$ using 
Monte Carlo or Quasi-Monte Carlo (see \cite{oates2016control} for a discussion in the context of RKHS). We found this approach to work well
in practice in high-dimensions ($D = 15$), even when $m$ was fixed, meaning
that the scaling was effectively independent of the dimension $D$.

Using the Sobolev kernel in Sec.~\ref{sec:sobolev},
we compared the exact calculation of $\tilde K_{\bf uu}$ with $s=1$, $a=10$, and $\gamma = .5$ to our approximate calculation. For illustration we compared a coarse grid of size 10 on the unit interval (left) to a finer grid of size 100. The RMSE was 2E-3 for the coarse grid and 1.6E-5 for the fine grid, as shown in Fig.~\ref{fig:sobolev-comparison}.  
In the same figure we compared the exact calculation of $\tilde K_{\bf xx}$ with $s=1$, $a=10$, and $\gamma = .5$ to our Nystr\"om-based approximation, where $x_1, \ldots, x_{400} \sim \mbox{Beta}(.5,.5)$ distribution. The RMSE was 0.98E-3. A low-rank approximation using only the top $5$ eigenvalues gives the RMSE of 1.6E-2.
As Figure \ref{fig:sobolev-comparison}, demonstrates, good approximation is possible with a fairly coarse grid ${\bf u}=\left(u_1,\ldots,u_m\right)$ as well as with a low-rank approximation.
\begin{figure}[ht!]
  \centering
  \includegraphics[width=.5\paperwidth]{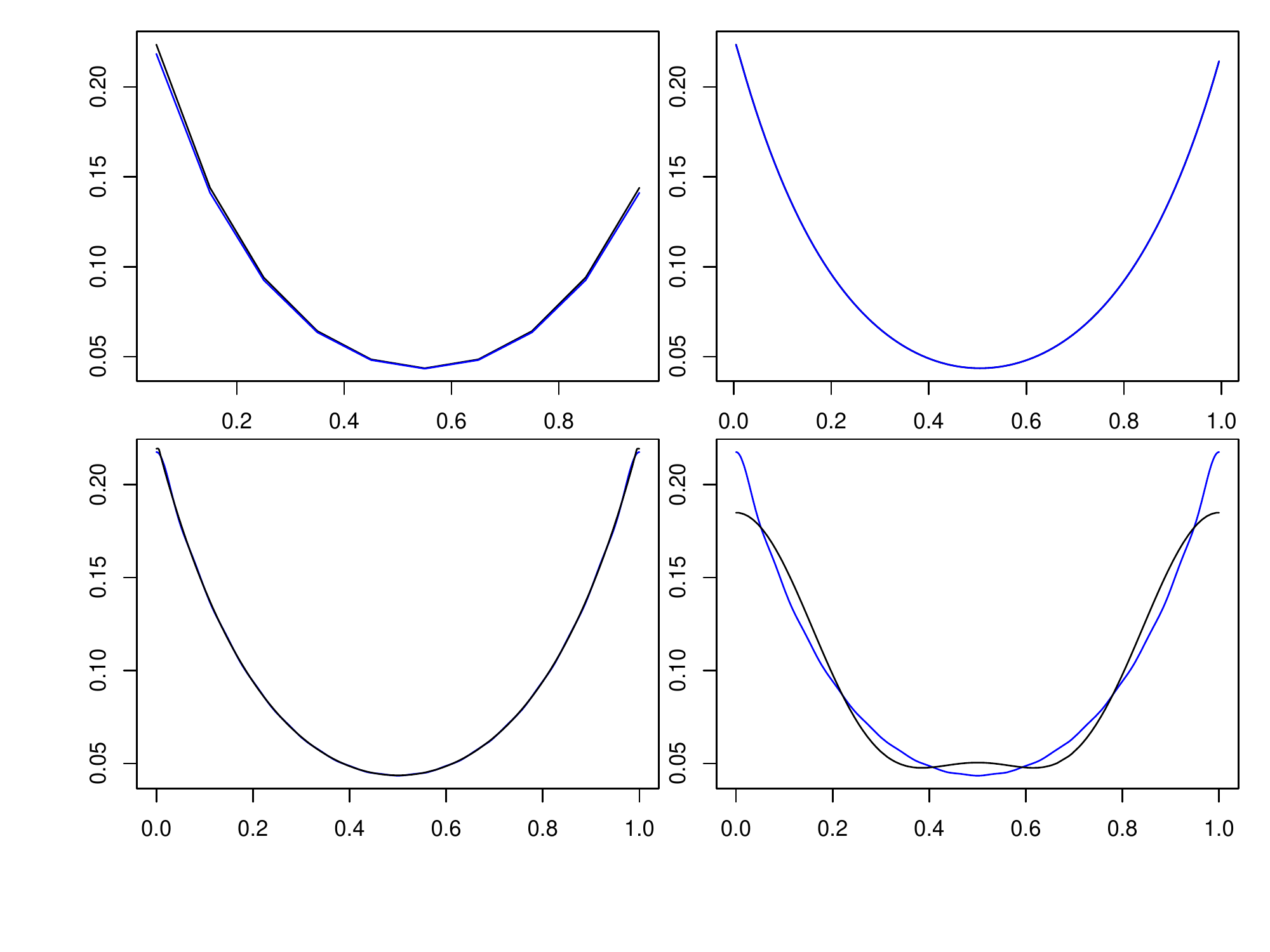}
	\caption{Using the Sobolev kernel in Sec.~\ref{sec:sobolev},
	    we compared the exact calculation of $\tilde K_{\bf uu}$ with $s=1$, $a=10$, and $\gamma = .5$ to our approximate calculation. For illustration we tried a coarse grid of size 10 on the unit interval (top left) to a finer grid of size 100 (top right). The RMSE was 2E-3 for the coarse grid and 1.6E-5 for the fine grid.
    We compare the exact calculation of $\tilde K_{\bf xx}$ with $s=1$, $a=10$, and $\gamma = .5$ to our Nystr\"om-based approximation, where $x_1, \ldots, x_{400} \sim \mbox{Beta}(.5,.5)$ distribution (bottom left). The RMSE was 0.98E-3. A low-rank approximation using only the top $5$ eigenvalues gives the RMSE of 1.6E-2 (bottom right).}
  \label{fig:sobolev-comparison}
\end{figure}

\section{Inference}
The penalized risk can be readily minimized with gradient descent.\footnote{While the objective is not convex,
in practice we observed very fast convergence, and stable results given random starting points.}
Let $\alpha = [\alpha_1, \ldots, \alpha_N]^{\top}$ and
$\tilde K$ be the Gram matrix corresponding to $\tilde k$ such that $\tilde K_{ij} = \tilde k(x_i,x_j)$. Then $[f(x_1), \ldots, f(x_N)]^{\top} = \tilde K \alpha$ and
the gradient of the objective function $J$ from \eqref{eq:Jtildenorm} is given by
\begin{align*}
	\nabla_{\alpha} J & = - \nabla_{\alpha} \sum_i \log(a f^2(x_i)) + \gamma \nabla_{\alpha} \|f \|_{\mathcal{H}_{\tilde k}}^2  \\
	& = - \nabla_{\alpha} \sum_i \log(a (\sum_j \tilde k_{ij} \alpha_j)^2)  + \gamma \nabla_{\alpha} \alpha^{\top} \tilde K \alpha \\
& = - \sum_i \frac{2 a (\sum_j \tilde k_{ij} \alpha_j) \nabla_{\alpha} \sum_j \tilde k_{ij} \alpha_j}{a (\sum_j \tilde k_{ij} \alpha_j)^2} + 2 \gamma \tilde K \alpha \\
&= - \sum_i \frac{ 2 \tilde K_{\cdot i}}{\sum_j \tilde k_{ij} \alpha_j} + 2 \gamma \tilde K \alpha \\
			  & = -2 \sum_i(\tilde K_{\cdot i} ./ (\tilde K \alpha))  + 2 \gamma \tilde K \alpha
\end{align*}
where $ ./ $ denotes element-wise division.
Computing $\tilde K$ requires $\mathcal{O}(N^2)$ time and memory, and each gradient and likelihood computation
requires matrix-vector multiplications which are also $\mathcal{O}(N^2)$. Overall, the running time is $\mathcal{O}(qN^2)$ for $q$ iterations of the gradient descent method, where $q$ is usually very small in practice.

\subsection{Hyperparameter selection}
\label{sec:hyperparameters}
Analogously to the classical problem of bandwidth selection in kernel intensity estimation
(e.g.~\citep{diggle1985kernel,berman1989estimating,brooks1991asymptotic}), some criteria must
be adopted in order to select hyperparameters of the kernel $k$ and also $\gamma$ and $a$. We suggest crossvalidating on the 
negative log-likelihood of the inhomogeneous Poisson process (i.e.~before we introduced the penalty term) from Eq.~\eqref{eq:log-likelihood}.
The difficulty with this approach is that we must deal with the integral $\int_S f^2(u) du$
of the intensity over the domain, which, for our model $f(\cdot)=\sum_{j=1}^N \alpha_j \tilde k(x_j,\cdot)$ is generally intractable.
As an approximation, we suggest either grid or Monte Carlo integration. Recall
that in Section \ref{sec:numerical} we approximated $\tilde k$ using a set of 
locations ${\bf u}=(u_1,\ldots,u_m)$. We can reuse these points to approximate the integral:
\begin{equation}
\int_S f^2(u) du \approx \frac{1}{m} \sum_i f^2(u_i).
\end{equation}
As $f(u_i)={\tilde K}_{u_i{\bf x}}\alpha$, this approximation is given by $\frac{1}{m}\alpha^\top {\tilde K}_{{\bf xu}}{\tilde K}_{{\bf ux}}\alpha$.

\section{Na\"ive RKHS model}
In this section, we compare the proposed approach, which uses the representer theorem in the transformed kernel $\tilde k$, to the na\"ive one, where a solution to Eq.~\eqref{eq:objective} of the form $f(\cdot) = \sum_{j=1}^N \alpha_j k(x_j, \cdot)$ is sought even though the representer theorem in $k$ need not hold. Despite being theoretically suboptimal, this is a natural model to consider, and it might
perform well in practice. 

The corresponding optimization problem is:
\begin{equation}
	\min_{f\in \mathcal{H}_k} \left\{-\sum_{i=1}^N \log(af^2(x_i)) + a\int_{S} f^2(x) dx + \gamma \|f \|_{\mathcal{H}_k}^2\right\}
 \nonumber
\end{equation}

While the first and the last term are straightforward to calculate for any $f(\cdot) = \sum_j \alpha_j k(x_j, \cdot)$, $\int_{S} f^2(x) dx$ needs to be estimated. As in the previous section, we consider a uniform grid or set of sampled points
${\bf u}=(u_1,\ldots,u_m)$ covering the domain and use approximation 
\begin{equation}
    \int_{S} f^2(u) du \approx \frac{1}{m} \sum_i f^2(u_i) =  \frac{1}{n} \alpha^{\top} K_{\bf{xu}} K_{\bf{ux}} \alpha.
\end{equation}
The optimization problem thus reads:
\begin{equation}
 \min_{\alpha\in\mathbb R^N} \left\{-\sum_{i=1}^N \log(a(\alpha^{\top} K_{{\bf x}x_i})^2) +
	\alpha^\top \left(\frac{a}{n}K_{\bf{xu}} K_{\bf{ux}}+\gamma K_{\bf{xx}}\right) \alpha \right\}.
 \label{eq:wrong_objective}
\end{equation}
As above, the gradient of this objective with respect to $\alpha$ can be readily calculated, and optimized with gradient descent.

\section{Experiments}
We use cross-validation to choose the hyperparameters in our methods: $a$, the fixed intensity, 
$\gamma$, the roughness penalty, and the length-scale of the kernel $k$, minimizing
the negative log-likelihood as described in Section \ref{sec:hyperparameters}.

To calculate RMSE, we either make predictions at a grid of locations
and calculate RMSE compared to the true intensity at that grid or for the high-dimensional
synthetic example we pick a new uniform sample of locations over the domain and calculate
the RMSE at these locations. 
We used limited memory BFGS in all experiments involving optimization, and found that it converged very quickly
and was not sensitive to initial values. Code for our experiments is available at \url{https://github.com/BigBayes/kernelpoisson}.

\subsection{1-d synthetic Example}
We generated a synthetic intensity using the Mercer expansion of a SE kernel with lengthscale $0.5$, producing
a random linear combination of 64 basis functions, weighted with iid draws $\alpha \sim \mathcal{N}(0,1)$.
In Fig.~\ref{fig:squared-exp} we compare ground truth to estimates made with: our RKHS method with SE kernel, the na\"ive RKHS approach with SE kernel,
and classical kernel intensity estimation with bandwidth selected by crossvalidation.
The results are typical of what we observed on 1D and 2D examples: given similar kernel choices, each method performed similarly,
and numerically there was not a significant difference in terms of the RMSE compared to the true underlying intensity.
\begin{figure}[h]
	\centering
	\captionbox{\small A synthetic dataset, comparing our RKHS method, the na\"ive model, and kernel smoothing to a synthetic intensity ``true''. The rug plot at bottom gives the location of points in
	the realized point pattern. The RMSE for each method was similar.\label{fig:squared-exp}}{%
	\hspace{.15in}\includegraphics[width=.37\paperwidth]{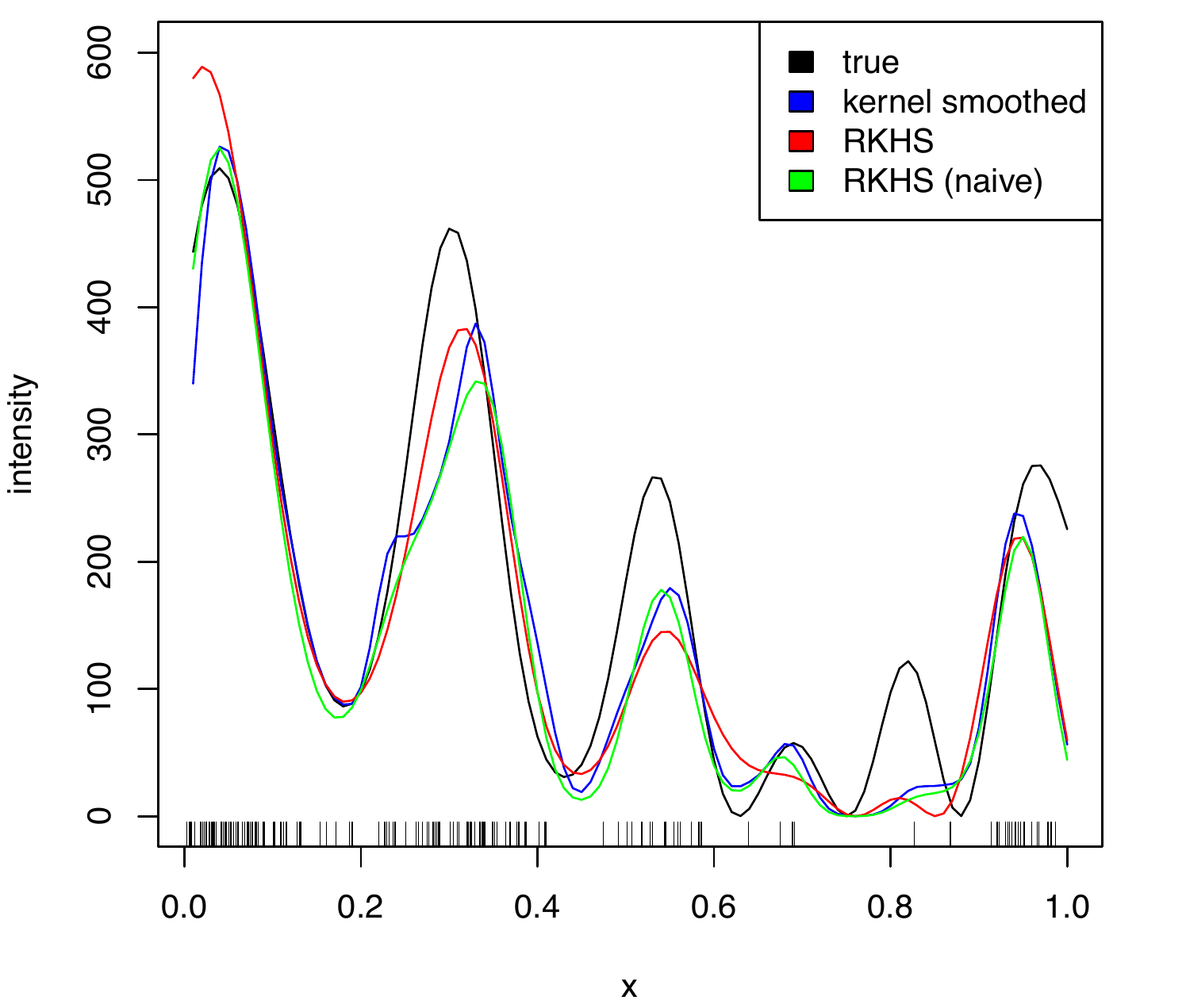}
    }
  
\end{figure}

\begin{figure}[ht!]
  \begin{subfigure}[t]{2.5in} \centering \includegraphics[width=.28\paperwidth]{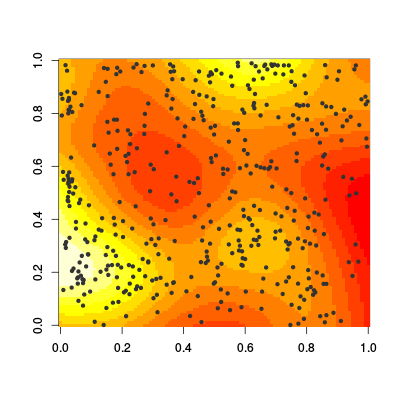}
      \caption{\small KIE with edge correction}\label{fig:kde}
  \end{subfigure}%
  \begin{subfigure}[t]{2.5in} \centering \includegraphics[width=.28\paperwidth]{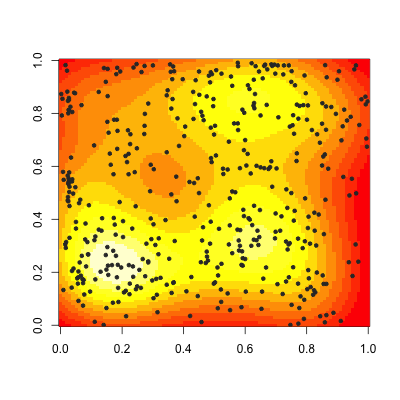}
      \caption{\small KIE without edge correction}\label{fig:kde}
  \end{subfigure}

    \begin{subfigure}[t]{2.5in} \centering \includegraphics[width=.28\paperwidth]{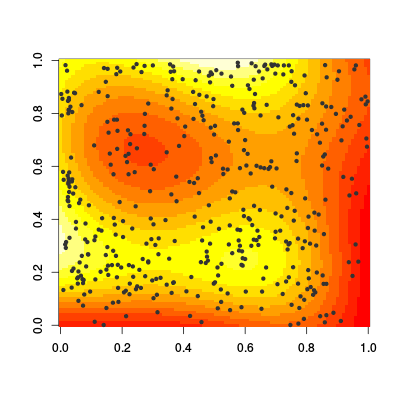}
        \caption{\small Our RKHS method with $\tilde k$}\label{fig:rkhs}
    \end{subfigure}%
    \begin{subfigure}[t]{2.5in} \centering \includegraphics[width=.28\paperwidth]{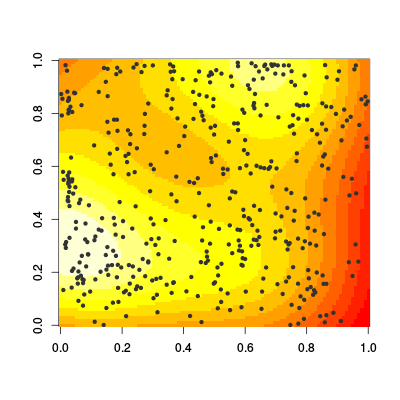}
        \caption{\small Na\"ive RKHS method}\label{fig:naive}
    \end{subfigure}%
    \caption{\small Location of white oak trees in Lansing, Michigan, smoothed with various approaches. Squared exponential
    kernels are used throughout. Edge correction makes a noticeable difference for classical kernel
    intensity estimation. Comparing (a) and (c) it is clear that our method is automatically performing edge correction.}
  \label{fig:whiteoak}
\end{figure}

	\begin{table}[h]
	\caption{\small Tree Point Patterns from R Package {\tt spatstat}}
	\label{table:environmental}
	\footnotesize
	\centering
\begin{tabular}{|c|c|c|c|}
	\hline
	Dataset & 	Kernel intensity estimation & Na\"ive approach  & Our approach with $\tilde k$ \\\hline		
	Lansing: Black oak (n = 135) & 234 & 233 & 227 \\ \hline
	  Hickory (n = 703) & 1763 & 1746 & 1757 \\ \hline
	  Maple (n = 514) & 1239 & 1228 & 1233 \\ \hline
	  Misc (n = 105) & 179 & 177 & 172 \\ \hline\hline
	  New Zealand (n = 86) & 119 & 119 & 119 \\ \hline
	  Red oak (n = 346) & 726 & 726 & {\bf 739} \\ \hline
	  Redwoods in California (n = 62) & 79 & 84 & 77 \\ \hline
	  Spruces in Saxonia (n = 134) & 215 & 212 & 212 \\ \hline
	  Swedish pines (n = 71) & 91 & 89 & 90 \\ \hline
	  Waka national park (n = 504) & 1142 & 1141 & 1144 \\ \hline
	  White oak (n = 448) & 992 & 992 & 996 \\ \hline

\end{tabular}
\end{table}

\subsection{Environmental datasets} 

Next we demonstrate our method on a collection of two-dimensional environmental datasets giving
the locations of trees. Intensity estimation is a standard first step in both exploratory analysis and modelling of these types
of datasets, which were obtained from the R package {\tt spatstat}. We calculated the intensity using
various approaches:
our proposed RKHS method with $\tilde k$ with a squared exponential kernel, the na\"ive RKHS method with squared exponential kernel,
and classical kernel intensity estimation (KIE) with edge correction. Each method used a squared exponential kernel.
We report average held-out cross-validated likelihoods in Table \ref{table:environmental}. With the exception of our method performing better on the Red oak dataset,
each method had comparable performance. It is interesting to note, however, that our method does not require any 
explicit edge correction\footnote{Because no points are observed outside the window $S$, intensity estimates near the edge are biased downwards \citep{jones1993simple}.}, because we are optimizing a likelihood which explicitly takes into account the window.
A plot of the resulting intensity surfaces for each method
and the effect of edge correction are shown in Fig.~\ref{fig:whiteoak} for the Black oak dataset.

\subsection{High dimensional synthetic examples} 

We generated random intensity surfaces in the unit hypercube for dimensions $D
= 2, \ldots, 15$.  The intensity was given by a constant multiplied by the
square of the sum of 20 multivariate Gaussian pdfs with random means and
covariances. The constant was automatically adjusted so that the number of
points in the realizations would be held close to constant, in the range 190-210.  We
expected this to be a relatively simple synthetic example for kernel intensity
estimation with a Gaussian kernel in low dimensions, but not in high
dimensions.  From each random intensity, we generated two random realizations,
and trained our model using 2-fold crossvalidation with these two datasets. We
predicted the intensity at a randomly chosen set of points and calculated the
mean squared error as compared to the true intensity. For each dimension we
repeated this process 100 times comparing kernel intensity estimation, the
na\"ive approach, and our approach with $\tilde k$.  

Using the same procedure, but a sum of 20 multivariate Student-t distributions
with 5 degrees of freedom, random means and covariances, and number of points
in the realizations ranging from 10 to 1000, we generated 500 random surfaces,
with dimension $D = 2, \ldots, 25$. We expected this to be a difficult synthetic
example for all of the methods due to a potential for model misspecification,
as we continue to use squared exponential kernels, but the intensity surface
is non-Gaussian.

As shown in Fig.~\ref{fig:rkhs-vs-kde-vs-naive}(a) once we reach dimension 9 and above, our
RKHS method with $\tilde k$ begins to outperform kernel intensity estimation,
where performance is measured as the fraction of times that the MSE is smaller across
100 random datasets for each $D$.  Our method
also significantly outperforms the na\"ive RKHS method as shown in
Fig.~\ref{fig:rkhs-vs-kde-vs-naive}(b).  For $D = 15$ the difference
between the two RKHS methods is not significant. This could be due to the
fact that the number of points in the point pattern remains fixed, so the
problem becomes very hard in high dimensions.\footnote{Note that our
experiments are sensitive to the overall number of points in the synthetic
point patterns; since kernel density estimation is a consistent method
\citep{wied2012consistency}, we should expect kernel intensity estimation to
become more accurate as the number of points grows. However, consistency in the
sense of classical statistics is not necessarily useful in point processes,
because our observations are not iid; the number of points that we observe is
in fact part of the dataset since it reflects the underlying intensity.}
As shown in the Fig.~\ref{fig:rkhs-vs-kde-vs-naive}(c), kernel
intensity estimation is almost always better than the na\"ive RKHS approach,
although the difference is not significant in high dimensions.

For the Student-t experiment, as shown in Fig.~\ref{fig:rkhs-vs-kde-vs-naive}(d)-(f), 
our RKHS method always outperforms kernel intensity estimation and is better than the na\"ive
method in dimensions below $D = 20$. To assess the amount of improvement, rather than just
its statistical significance, we compared the percent improvement in terms of MSE gained by
our method versus the competitors, just focusing on $D = 10$ in Fig.~\ref{fig:pct-improvement-studentt}.
On this metric (intuitively, ``how much do you expect to improve on average'') our
method shows reasonably stable results as compared to KIE, while the performance of the
na\"ive method is revealed to be very variable. Indeed, the standard deviation across the random
surfaces for $D=10$ of the MSE
was 56 for both our method and KDE but 166 for the na\"ive method, perhaps due to overfitting.

\subsection{Computational complexity} 
Using the synthetic data experimental setup, we evaluated the time complexity of our method
with respect to dimensionality $d$, number of points in the point pattern dataset $n$,
and number of points $s$ used to estimate $\tilde k$ (Fig.~\ref{fig:scaling-S}), confirming our theoretical
analysis.

The effect
of the dimensionality $d$ was negligible in practice, because the main
calculations rely only on an $n \times n$ Gram matrix whose calculation is relatively fast even for high dimensions.
Our method's time complexity scales as $\mathcal{O}(n^2)$ as shown in Fig.~\ref{fig:scaling-N} (but as discussed in Section \ref{sec:numerical},
a primal representation is available which would give linear scaling.) where we used $s = 200$ sample points to estimate $\tilde k$.
While a small $s$ worked well in practice, we investigated much larger values of $s$. As shown in Fig.~\ref{fig:scaling-S}
the time complexity scaled as $\mathcal{O}(s^2)$ where the number of points was fixed to be 150; note that we fixed the rank of the eigendecomposition to be 20.
\begin{figure}[h]
  \centering
  \includegraphics[width=.4\paperwidth]{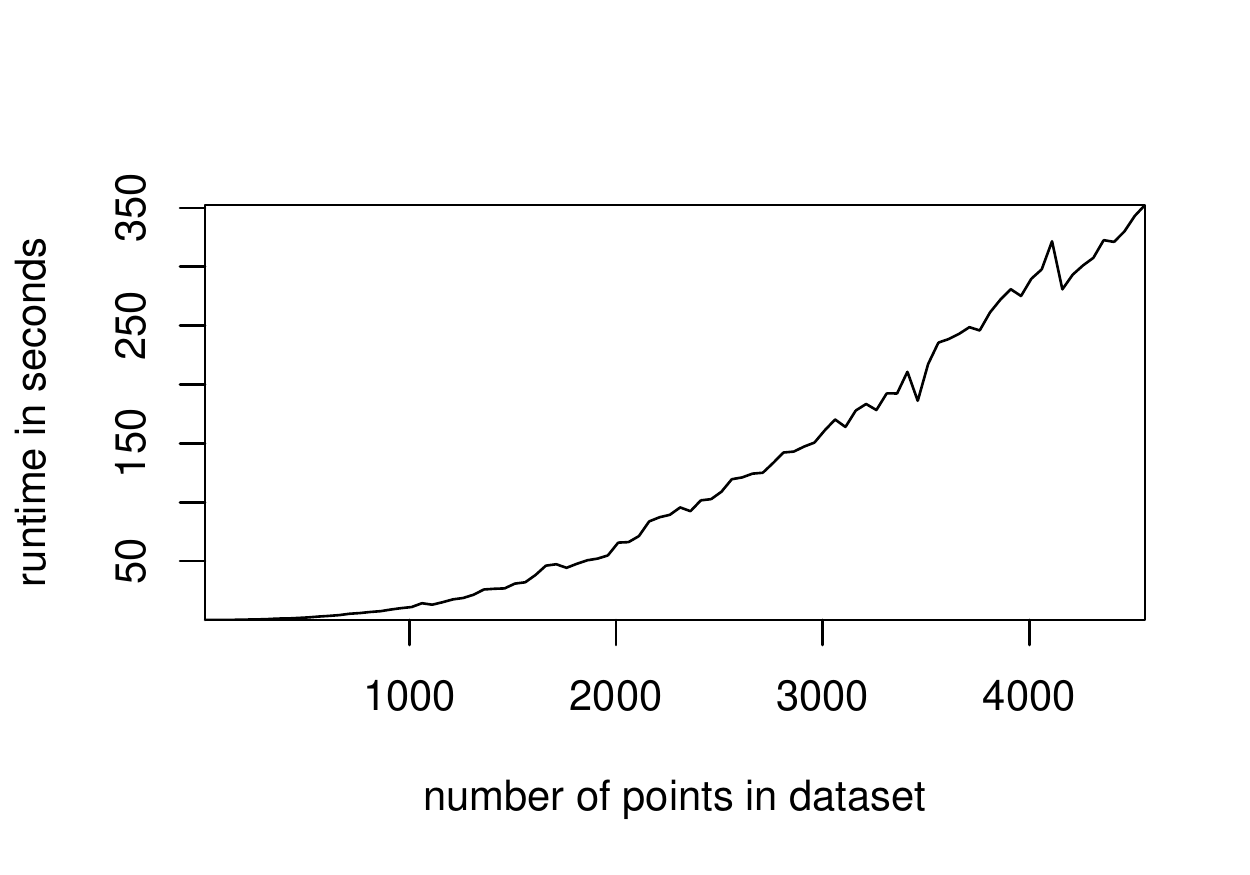}
  \caption{Run-time of our method versus number of points in the point pattern dataset.}
  \label{fig:scaling-N}
\end{figure}
\begin{figure}[h]
  \centering
  \includegraphics[width=.4\paperwidth]{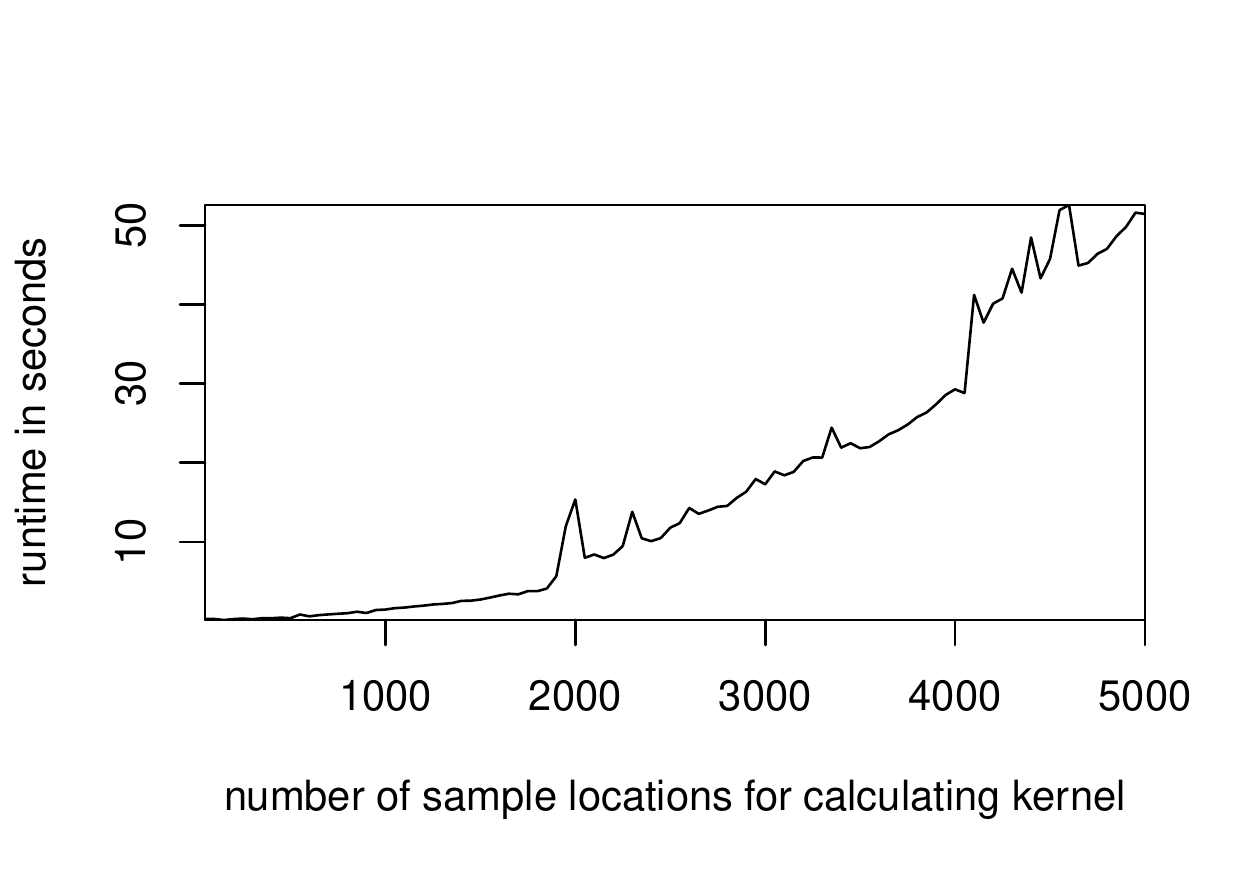}
  \caption{Run-time of our method versus number of sample points used to calculate $\tilde k$.}
  \label{fig:scaling-S}
\end{figure}
\newpage

\begin{figure}[h!]
	\includegraphics[width=.3\paperwidth]{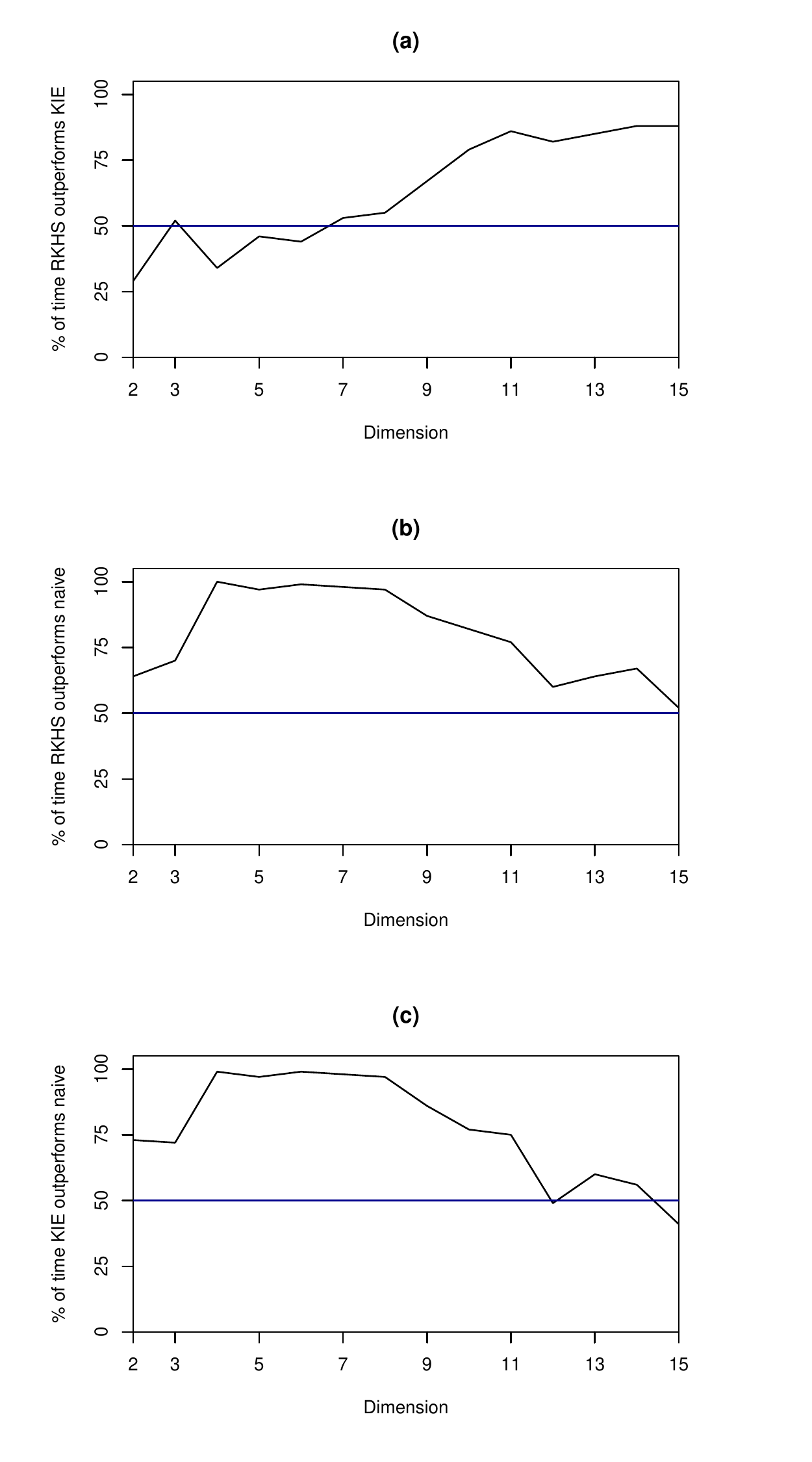}%
	\includegraphics[width=.3\paperwidth]{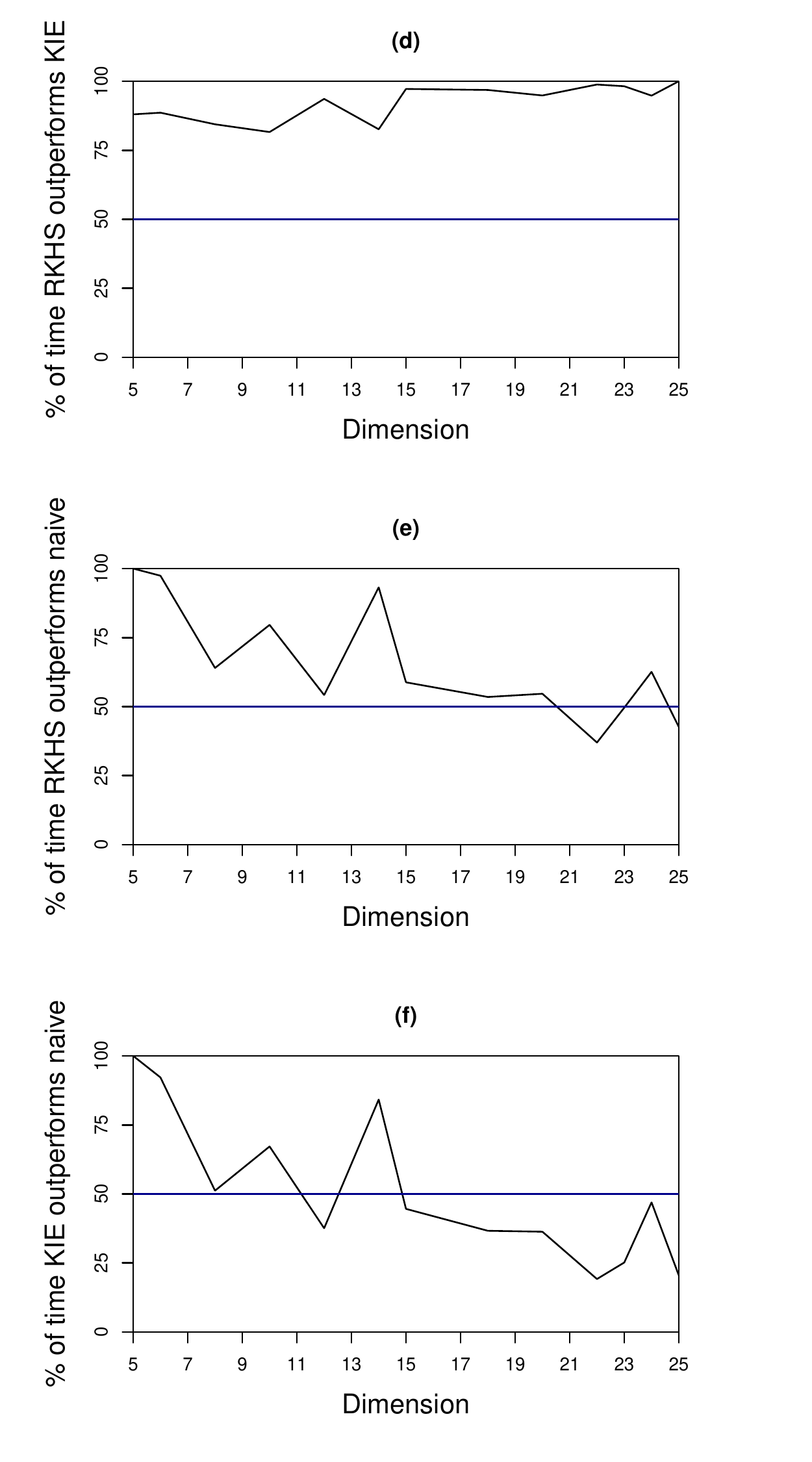}
	\caption{Three methods were compared: our RKHS method, the na\"ive RKHS method, and kernel intensity estimation,
		based on 100 random surfaces for each dimension $D$ in two experimental setups.
		In (a)-(c), the intensity surface was the squared sum of skewed multivariate Gaussians. In (d)-(f) the surface
		was a mixture of skewed multivariate Student-t distributions, with 5 degrees of freedom.  
		  In (a) and (d): comparison of our RKHS method versus KIE.
		In (b) and (e): our RKHS method versus the na\"ive RKHS method.
		In (c) and (f): comparison of KIE and the na\"ive RKHS approach.
		We used squared exponential kernels for all methods.
		In the Gaussian case (a)-(c), our method significantly outperforms kernel intensity estimation as the dimension increases,
		and outperforms the na\"ive method throughout.  Kernel intensity estimation almost always outperforms the na\"ive approach.  
		In the Student-t case (d)-(f), our method always outperforms kernel intensity estimation, and outperforms the na\"ive approach
		until very high dimensions. Neither kernel intensity estimation nor the na\"ive approach are consistently better than each
		other.}
	  \label{fig:rkhs-vs-kde-vs-naive}
\end{figure}

\begin{figure}[ht]
	\includegraphics[width=.5\paperwidth]{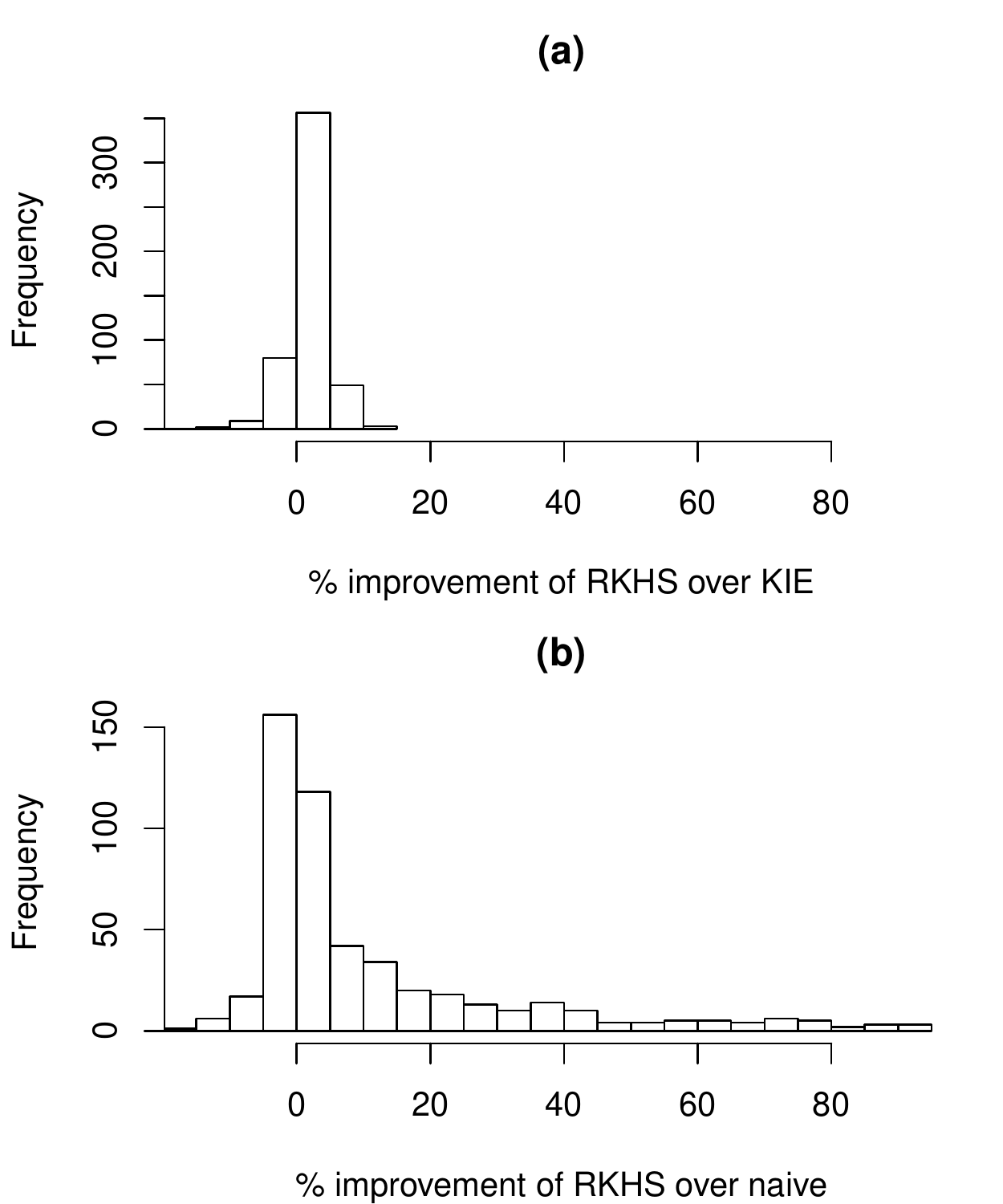}
	\caption{To understand the practical (as opposed to statistical) significance of the results in Fig.~\ref{fig:rkhs-vs-kde-vs-naive}(d)-(f), where we generated random surfaces by squaring the sums of multivariate Student-t distributions, we considered dimension $D = 10$, in which our RKHS method was better than both the na\"ive method and kernel intensity estimation (KIE) but there was not a significant difference between KIE and the na\"ive method, and calculated the percent improvement in terms of MSE comparing our RKHS method to the na\"ive method (left) and our RKHS method to and kernel intensity estimation (KIE) (right). The improvement of our method over KIE is apparent, albeit perhaps only modest in this example. Meanwhile, our method is sometimes quite a bit better than the na\"ive method, which is often very inaccurate.}
	\label{fig:pct-improvement-studentt}
\end{figure}

\begin{figure}[ht!]
\centering
\includegraphics[width=.6\paperwidth]{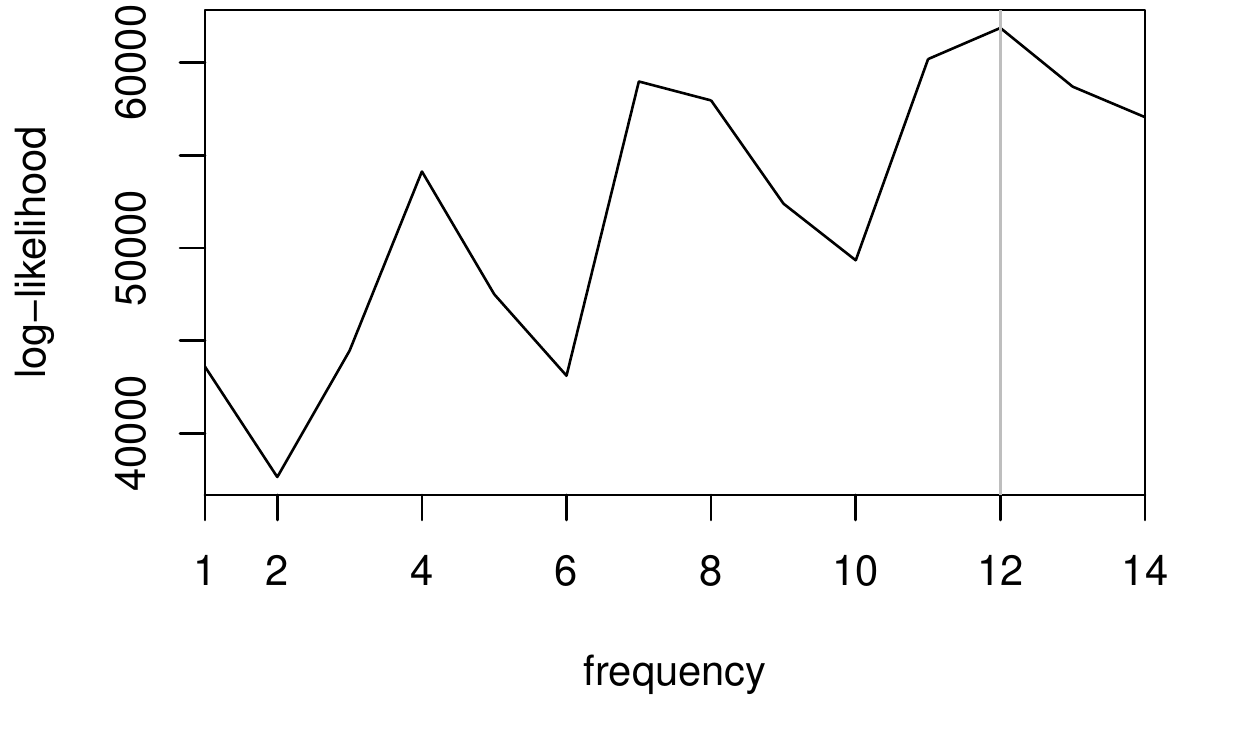}%
\caption{\small Log-likelihood for various frequencies of a periodic spatiotemporal kernel in a dataset of 18,441 geocoded, date-stamped theft events from Chicago, using our RKHS model. The dataset is for 12 weeks starting January 1, 2004, and the maximum log-likelihood is attained when the frequency is 12, meaning that there is a weekly cycle in the data. Results using the na\"ive model were less sensible,
	with a maximum at period 1 (indicating no periodicity), with periods 5 and 2 (corresponding to a 16.8 day cycle and a 42 day cycle) also having high likelihoods. }
\label{fig:crime}
\end{figure}

\subsection{Spatiotemporal point pattern of crimes} 
To demonstrate the ability to use domain specific kernels and
learn interpretable hyperparameters, we used 12 weeks (84 days) of geocoded, date-stamped reports of theft
obtained from Chicago's data portal (data.cityofchicago.org) starting January 1, 2004, a relatively large spatiotemporal point pattern consisting of 18,441 events. We used the following kernel: $\exp(-.5 s^2/\lambda_s^2) (\exp(-2\sin^2(t \pi p)) + 1) (\exp(-.5 t^2 / \lambda_t^2))$
which is the product of a separable squared exponential space and decaying periodic time kernel (with frequency $p$  in a time domain normalized to range from $0$ to $1$) plus a separable squared exponential space and time kernel. After finding reasonable values for the lengthscales and other hyperparameters of $\tilde k$ through exploratory data analysis, we used 2-fold cross-validation and calculated average test log-likelihoods for the number of total cycles $p$ in the 84 weeks $= 1, 2, \ldots, 14$ or equivalently a period of length 12 weeks (meaning no cycle), 6 weeks, ..., 6 days. These log-likelihoods are shown in Fig.~\ref{fig:crime}; we found that the most likely frequency is 12, or equivalently a period lasting 1 week. This makes sense given known day-of-week effects on crime.

\section{Conclusion}
We presented a novel approach to inhomogeneous Poisson process intensity
estimation using a representer theorem formulation in an appropriately
transformed RKHS, providing a scalable approach giving strong performance on
synthetic and real-world datasets. Our approach outperformed the classical
baseline of kernel intensity estimation and a na\"ive approach for which the
representer theorem guarantees did not hold. In future work, we will consider marked
Poisson processes and other more complex point process models, as well as
Bayesian extensions akin to Cox process modeling. 

\bibliographystyle{imsart-nameyear}
\bibliography{biblio}
		
\end{document}